\documentclass[letterpaper,10pt,conference]{ieeeconf}

\IEEEoverridecommandlockouts
\usepackage{amsmath,amssymb,amsfonts,mathtools}
\usepackage[table]{xcolor}

\usepackage{graphicx}
\usepackage{subcaption}
\usepackage{cite}
\usepackage{hyperref}
\usepackage{algorithm}
\usepackage{algpseudocode}
\usepackage{booktabs}
\usepackage{xcolor}

\newtheorem{prop}{Proposition}
\newtheorem{lemma}[prop]{Lemma}
\newtheorem{thm}[prop]{Theorem}

\begin{document}

\title{\LARGE \bf
Distribution-Transfer Safe-Horizon MPC under Mode Uncertainty
}

\author{Stephen Crawford%
\thanks{Department of Computer Science, Brown University, Providence, RI, USA.
{\tt\small stephen\_crawford@brown.edu}}%
\and
Nora Ayanian%
\thanks{Department of Computer Science, Brown University, Providence, RI, USA.
{\tt\small nora\_ayanian@brown.edu}}%
}

\maketitle

\begin{abstract}
Scenario-based MPC is an attractive strategy for chance-constrained motion planning that approximates uncertainty via a finite set of sampled scenarios. As a sampling-based method, scenario-based MPC is sensitive to distribution mismatch. We address this problem in the context of Safe-Horizon Model Predictive Control (SH-MPC) with obstacles governed by switching dynamic modes. From finite mode observations, we construct a confidence set for the unknown categorical mode law and derive a multiplicative domination bound that transfers a Safe-Horizon collision-risk certificate from a selected scenario-sampling distribution to every law in the confidence set. Wasserstein geometry is used to regularize probability reallocation among modes according to the similarity of their induced trajectory predictions, while a collision-risk surrogate biases sampling toward dangerous modes. The resulting certificate explicitly quantifies the additional tightening required under distribution mismatch and exposes the multiplicative conservatism that arises when several obstacle-wise transfer factors are combined.
\end{abstract}


\section{Introduction}
Scenario-based optimization provides a tractable surrogate for chance-constrained control by enforcing constraints on finitely many sampled realizations of uncertainty \cite{Cal2006ScenarioApproach,Ca2018ScenarioTheory}. In MPC, this approach has been used to reduce conservatism relative to analytic convex approximations while preserving real-time feasibility \cite{de2021ScenarioBasedTraj,degroot2023SBMotionBounded, Shang2021SceMPC, Jekl2025SceVelPlanning}. Scenario-based optimization is used as the basis of Safe-Horizon Model Predictive Control (SH-MPC), which probabilistically bounds the joint probability of collision over a trajectory 
\cite{degroot2023SBMotionBounded}. 

Despite the promising potential of SH-MPC, it shares the limitations of any sampling-based approach. As a bound derived from traditional scenario theory \cite{Cal2006ScenarioApproach}, the Safe-Horizon joint collision-risk bound only applies with respect to the distribution of chance constraints that were sampled. Consequently, when distribution mismatch exists, we lack guarantees with respect to the true distribution. This problem exists across scenario-based MPC, but is particularly easy to see in the context of obstacles with switching dynamics. In cases where we must sample across multiple discrete dynamic modes, scenario generation requires a probability distribution over the horizon-level behavior mode used to generate each obstacle trajectory. The same transfer argument can be applied more broadly to scenario-optimization-based MPC formulations when certification is defined with respect to a known sampling law. Other relevant works include per-step chance-constraint formulations \cite{de2021ScenarioBasedTraj}, candidate-path scenario representations \cite{Shang2021SceMPC}, and reinforcement-learning-augmented scenario-based MPC \cite{Jekl2025SceVelPlanning}.

Our approach first constructs a finite-sample confidence set for the unknown categorical mode law. From this set we derive coordinate-wise probability envelopes that yield a multiplicative change-of-measure bound. We then select a scenario-sampling distribution \(\mathbf q^\star\) using a Wasserstein-constrained risk-aware optimization that prevents modes classified as dangerous from being downweighted relative to nominal sampling. A Safe-Horizon certificate computed under \(\mathbf q^\star\) can then be transferred, after an explicit tightening by the domination factor, to every mode law in the finite-sample confidence set.

Our construction differs from conventional Wasserstein distributionally robust MPC, where ambiguity is typically embedded directly into the control problem through worst-case expectations, risk measures, or robust chance constraints \cite{mohajerin2018wasserstein,kuhn2024WDRO,Hakobyan2020WDRO, Zhong2021DDRMPC,gao2023drso,Aolaritei2023WTMPC, Zolanvari2023WDRRC, Micheli2022DDRMPC}. We do not optimize a worst-case expected cost or chance constraint over a Wasserstein ball. Instead, Wasserstein geometry regularizes the redistribution of probability among discrete trajectory modes while robustness to categorical mode-probability mismatch is provided by the distribution-transfer certificate developed below.

Schuurmans and Patrinos' work \cite{schuurmans2023generalframeworklearningbaseddistributionally} is particularly relevant because it explicitly models uncertain transition probabilities in a Markov jump system and embeds the resulting ambiguity-induced risk measures into a multistage control problem. Our formulation instead makes the narrower assumption of independent horizon-level categorical mode observations over the calibration interval. This permits a finite-sample categorical confidence construction while leaving the downstream SH-MPC problem unchanged; extension to state-conditioned or Markov mode laws is left for future work.

We make three main contributions. First, we construct a finite-sample confidence set for the unknown categorical mode law. Second, we derive a multiplicative domination certificate that transfers Safe-Horizon violation guarantees from a selected sampling law to every mode law in this confidence set. Third, we introduce a Wasserstein-constrained risk-aware rule for choosing the scenario-sampling distribution while explicitly controlling its domination factor. We also empirically quantify certificate tightness, closed-loop performance, and computational overhead.


\section{Problem Formulation}\label{sec:problem}

We consider SH-MPC for a mobile robot navigating among dynamic obstacles with discrete behavior modes such as turning, maintaining velocity, or stopping. This setting provides a natural example of mode-probability mismatch in scenario-based MPC.

At each time step $t$, SH-MPC solves a receding-horizon optimization problem over horizon $N$, while formulating collision avoidance as a joint chance constraint over all obstacles and prediction steps in the horizon. The optimization uses a finite set of $S$ sampled obstacle-trajectory scenarios with each scenario representing a joint realization of the predicted trajectories of all obstacles. The general scenario program is given in Problem 3 of the original SH-MPC paper \cite{degroot2023SBMotionBounded}; we omit its standard MPC objective and constraint details here and focus on the uncertainty model relevant to our contribution.

To treat obstacles with switching dynamic modes, we incorporate a mode predictor into the scenario-program framework. Each obstacle $v \in \{1,\dots,V\}$ is assumed to have a known finite dynamic-mode set indexed as $\mathcal{M}_v=\{1,\dots,M_v\}$. For every obstacle and timestep, the mode predictor maintains a nominal belief $\widehat p_t^{v} \in \Delta_{\mathcal{M}_{v}}$ over the given obstacle's modes. This nominal belief follows a Dirichlet-multinomial posterior-predictive mean. Let \(n_{t,m}^{v}\) denote the number of observations of mode \(m\) for obstacle \(v\) available at time \(t\). Then, under a symmetric Dirichlet prior with concentration parameter \(\alpha_{\mathrm{KT}}>0\), the predictive likelihood of mode $m$ being obstacle $v$'s active mode at time $t$ is: 
\begin{equation}\label{eq:nominal_belief}
\widehat p_{t,m}^{v}
=
\frac{n_{t,m}^{v}+\alpha_{\mathrm{KT}}}{\sum_{j=1}^{M_{v}} n_{t,j}^{v}+ M_{v} \alpha_{\mathrm{KT}} },
\qquad
m\in\mathcal{M}_{v}.
\end{equation} We use the Krichevsky--Trofimov prior \cite{Krichevsky1981ThePO}, corresponding to $
\alpha_{\mathrm{KT}}=\frac{1}{2}$. While the predictor maintains a time-varying estimate
\(\widehat p_t^v\), we assume a fixed but unknown categorical law
\(p^{v,\star}\in\Delta_{\mathcal M_v}\) governing the horizon-level
mode observed at replanning instants during the calibration interval.
This law is assumed stationary over that interval.

After having predicted the active mode at $t$, we use a trajectory predictor to complete our switching dynamics setup. Conditioned on a mode $m \in \mathcal{M}_{v}$, the trajectory predictor generates a trajectory rollout over the prediction horizon using the stochastic affine Gaussian dynamics
\begin{equation}\label{eq:trajectory_predictor}
x_{k+1 + t|t}^v = A_m^v x_{k+t|t}^v + b_m^v + G_m^v w_{k+t|t}^v,
\;
w_{k+t|t}^v \sim \mathcal{N}(0,I),
\end{equation}
where $x_{k+t|t}^v=(p_{x, k+t|t}^v,p_{y, k+t|t}^v,\mathrm{v}_{x,k+t|t}^v,\mathrm{v}_{y, k+t|t}^v)\in\mathbb{R}^4$ is the obstacle's planar kinematic state at step $k$, $(A_m^{v},b_m^{v},G_m^{v})$ are the mode-dependent dynamics parameters, and $w_{k+t|t}^v$ is noise. Notably, we do not assume an explicit transition model for the mode-switching at each replanning step and we hold modes fixed throughout the prediction horizon.

Conditioned on mode \(m\) and initial state $x_t^v$, we denote the probability law over obstacle \(v\)'s predicted horizon trajectory induced by the dynamics in \eqref{eq:trajectory_predictor} as $K_m^v(\cdot\mid x_t^v)$. Holding the current state and replanning time fixed and abbreviating $K^v_m(\cdot) :=K^v_m(\cdot \mid x_t^v)$ we see that for any trajectory of obstacle $v$ at time $t$ $\xi_t^v$, measurable trajectory set \(A\), and realized mode $M_t^v$,
\begin{equation}\label{eq:proability_law_over_traj_set}
K_m^v(A)
:=
\mathbb P\!\left(
\xi_t^v\in A
\mid
M_t^v=m,x_t^v
\right).
\end{equation}
Thus, for any categorical mode distribution $p_t^v
=
(p_{t,1}^v,\ldots,p_{t,M_v}^v)
\in\Delta_{\mathcal M_v},
$
the corresponding trajectory law is the finite mixture
\[
\mathcal P_{p_t^v}^{\,v}
=
\sum_{m=1}^{M_v}
p_{t,m}^v K_m^v.
\]
With respect to $\widehat{p}^v_t$ and $p^{v, \star}$ we get,
\[
\mathcal P_{\widehat p_t^v}^{\,v}
=
\sum_{m=1}^{M_v}
\widehat p_{t,m}^v K_m^v,
\qquad
\mathcal P_{p^{v,\star}}^{\,v}
=
\sum_{m=1}^{M_v}
p_m^{v,\star}K_m^v.
\]
Evaluated on the trajectory set \(A\) this yields,
\begin{equation}\label{eq:trajectory_law}
\mathcal P_{\widehat p_t^v}^{\,v}(A)
=
\sum_{m=1}^{M_v}
\widehat p_{t,m}^v K_m^v(A).
\end{equation}
We assume the mode-conditioned laws \(K_m^v\) are correctly
specified, so nominal and true models differ only in their
categorical mixture weights. The resulting certificate therefore
addresses mode-probability mismatch and does not provide robustness
to misspecification of the mode-conditioned trajectory models.

For our discrete mode-probability model, we assume that obstacle mode selections and mode-conditioned trajectory realizations are conditionally independent across obstacles. Thus, the joint nominal trajectory law factorizes as the product measure
\begin{equation}\label{eq:joint_nominal_traj_law}
\mathcal P_{\widehat{\mathbf p}_t}
=
\bigotimes_{v=1}^{V}
\mathcal P_{\widehat p_t^v}^{\,v},
\end{equation}
where $
\widehat{\mathbf p}_t
:=
(\widehat p_t^1,\ldots,\widehat p_t^V)
$
denotes the collection of obstacle-wise nominal beliefs. Similarly, $
\mathbf p^\star
:=
(p^{1,\star},\ldots,p^{V,\star}),
\qquad
\mathcal P_{\mathbf p^\star}
:=
\bigotimes_{v=1}^{V}
\mathcal P_{p^{v,\star}}^{\,v}.
$ 

Our scenario-generation procedure mirrors this factorization by independently sampling the obstacle-wise modes and conditional trajectory realizations. Algorithmically, generating a trajectory from $\mathcal{P}_{\widehat{p}_t^v}^v$ is performed in two steps. First, a mode $M^{v, (i)} \sim \operatorname{Categorical}(\widehat{p}^v_t)$ is sampled. Second, a trajectory realization is sampled $\xi_t^{v, (i)} \sim K_{M^{v, (i)}}^v$ with
$
\xi_t^{v,(i)}
:=
\left(
x_{t+1|t}^{v,(i)},
\ldots,
x_{t+N|t}^{v,(i)}
\right)$, and 
$
\xi_t^{(i)}
:=
\left(
\xi_t^{1,(i)},\ldots,\xi_t^{V,(i)}
\right)$. Thus $
\xi_t^{(i)}$ represents a single joint-trajectory scenario at time $t$. Our two-stage procedure is equivalent to direct sampling from the induced mixture trajectory law: for any measurable trajectory set \(A\),
\[
\begin{aligned}
\mathbb P(\xi^v_t \in A)
&=
\sum_m
\mathbb P(M^v_t =m)\,
\mathbb P(\xi^v_t\in A\mid M^v_t=m, x_t^v)\\
&=
\sum_m
\widehat p_{t,m}^v K_m^v(A)\\
&=
\mathcal P_{\widehat p_t^v}^{\,v}(A).
\end{aligned}
\] 

Having explained our modification of the original SH-MPC setup, we now restate the Safe-Horizon certification and define the distribution mismatch problem. Let \begin{equation} \label{eq:scenario_set}
    \widehat{\Xi}_t = \left(\xi_t^{(1)}, \dots, \xi_t^{(S)}\right), \quad \xi_t^{(i)} \overset{\text{iid}}{\sim} \mathcal{P}_{\widehat{\mathbf p}_t}
\end{equation}
 denote the random scenario set at time $t$. Then let $\Theta$ denote the SH-MPC decision space and \(\Theta_{\xi_t^{(i)}}\subseteq\Theta\) denote the set of decisions satisfying all collision-avoidance constraints induced by scenario \(\xi_t^{(i)}\) over the prediction horizon. We define $\vartheta^\star_{\widehat{\Xi}_t}$ to be the (sub)optimal decision returned by SH-MPC given $\widehat{\Xi}_t$. Next define the violation probability with respect to joint trajectory law $\mathcal{P}$ and decision $\vartheta$ as $\mathrm{V}_{\mathcal{P}}(\vartheta):= \mathbb{P}_{\xi^{(i)}_t \sim \mathcal{P}}[ \vartheta \notin \Theta_{\xi_t^{(i)}}]$. Thus $\mathrm{V}_{\mathcal P}:\Theta\to[0,1]$ gives the probability that a new scenario drawn from \(\mathcal P\) causes a decision to violate its associated constraints. 

Following the nonconvex scenario-theoretic result of the source paper \cite{degroot2023SBMotionBounded}, let \(n_{\widehat\Xi_t}\)
denote the support size of
\(\vartheta^\star_{\widehat\Xi_t}\), let
\(\epsilon_S(\cdot)\) denote the scenario-theoretic violation bound
for \(S\) samples and the fixed scenario-removal budget used by the controller, and let \(\beta_{\mathrm{SH}}\) denote its failure
probability. Then

\[
\mathbb P_{
\widehat\Xi_t\sim
(\mathcal P_{\widehat{\mathbf p}_t})^{\otimes S}
}
\left[
\mathrm{V}_{\mathcal P_{\widehat{\mathbf p}_t}}
(\vartheta_{\widehat\Xi_t}^{\star})
>
\epsilon_S(n_{\widehat\Xi_t})
\right]
\le
\beta_{\mathrm{SH}}.
\] is a joint-bound on the violation probability of the decision given the entire scenario set at time $t$. 

As detailed in the original publication \cite{degroot2023SBMotionBounded}, Safe-Horizon
chooses a support limit \(\bar n\) such that
\[
\epsilon_S(\bar n)\le\epsilon.
\]
Since \(\epsilon_S(n)\) is nondecreasing in \(n\), any solution
satisfying \(n_{\widehat\Xi_t}\le\bar n\) also satisfies
\(\epsilon_S(n_{\widehat\Xi_t})\le\epsilon\). Therefore,
\[
\mathbb P
\left[
\mathrm{V}_{\mathcal P_{\widehat{\mathbf p}_t}}
(\vartheta_{\widehat\Xi_t}^{\star})
>\epsilon,\;
n_{\widehat\Xi_t}\le\bar n
\right]
\le
\beta_{\mathrm{SH}}.
\]

As can be seen, this certification is reliant on the distribution producing the scenarios. In mobile robotics, nearby obstacles may switch between distinct behavior patterns that are difficult to capture perfectly with predictive models. Therefore, the controller must rely on estimated categorical probabilities over a predefined mode library. Consequently, the nominal belief \(\widehat{\mathbf p}_t\) need not equal the unknown true mode distribution \(\mathbf p^\star\). This is what we term the distribution mismatch problem. In that case, the Safe-Horizon certificate remains valid with respect to the nominal trajectory law from which the scenarios were drawn, but does not by itself imply the same bound under the true trajectory law:  \begin{equation} \label{eq:distribution_mismatch}\mathrm{V}_{\mathcal{P}_{\widehat{\mathbf{p}}_t}}(\vartheta^\star_{\widehat{\Xi}_t}) \leq \epsilon \not\Rightarrow \mathrm{V}_{\mathcal{P}_{\mathbf p^\star}}(\vartheta^\star_{\widehat{\Xi}_t}) \leq \epsilon.
\end{equation} 

Given a nominal categorical belief \(\widehat{\mathbf p}_t\), finite calibration data, target true-distribution violation level \(\epsilon\), confidence-set failure probability \(\beta_{\rm DRO}\), and Safe-Horizon certificate failure probability \(\beta_{\rm SH}\), our goal is to construct finite-sample confidence sets \(\mathcal C_t^v\) and reweighted mode distributions
\[
\mathbf q_t^\star
=
(q_t^{1,\star},\ldots,q_t^{V,\star})
\] with corresponding joint trajectory law
\[
\mathcal P_{\mathbf q_t^\star}
:=
\bigotimes_{v=1}^{V}
\mathcal P_{q_t^{v,\star}}^{\,v}.
\]
such that: (i) the unknown joint mode law is contained in
\(\mathcal C_t=\mathcal C_t^1\times\cdots\times\mathcal C_t^V\)
with probability at least \(1-\beta_{\rm DRO}\); (ii) a Safe-Horizon certificate obtained under the reweighted sampling law
\(\mathcal P_{\mathbf q_t^\star}\) transfers to every
\(\mathbf p\in\mathcal C_t\); and (iii) Wasserstein geometry and a collision-risk surrogate are used to select among certificate-valid sampling laws while explicitly controlling the resulting domination factor.


\section{Distribution-Transfer}\label{sec:dist_transfer}
Having established the basic distribution mismatch problem, we now present the theory related to our distribution-transfer contribution.

\subsection{Confidence polytope}
We first construct a finite-sample confidence polytope \(\mathcal C_t^v\) for \(p^{v,\star}\) which contains $p^{v, \star}$ with probability at least $1 - \beta_{\mathrm{DRO}}^v$.

For obstacle \(v\) at time \(t\), let \(g_t^v\ge1\) denote the
number of observed calibration-mode realizations and assume
\[
Y_{t,1}^v,\ldots,Y_{t,g_t^v}^v
\overset{\mathrm{iid}}{\sim}
\operatorname{Categorical}(p^{v,\star}).
\] Given the $g^v_t$ draws, the observed count for mode $m$ is given by $n_{t,m}^v = \sum_{l = 1}^{g^v_t} \mathbf{1}\{Y^v_{t, l} = m\},$ meaning the individual mode counts follow a binomial \[
n_{t,m}^v \sim \operatorname{Binomial}(g_t^v, p_m^{v, \star}).
\] Accordingly, the finite-sample coverage result below does not extend
to temporally dependent observations, state-conditioned categorical
laws, mode-classification error, or misspecified mode-conditioned
trajectory kernels without additional modeling assumptions.

Given this setup, we form a confidence region using Clopper-Pearson confidence bounds \cite{clopper1934use}. Following our requirement that $\beta_{\mathrm{DRO}}^v >0, \quad \sum_{v=1}^V \beta_{\mathrm{DRO}}^v \leq \beta_{\mathrm{DRO}}$, we divide the failure probability across modes. 
This yields $
\delta_{v, m} = \frac{\beta^v_{\mathrm{DRO}}}{M_v}
$. 
Under the equal-tailed formulation, each tail receives $\frac{\delta_{v,m}}{2}$. For every $m$, we construct the marginal intervals $[L^v_{t,m}, U^v_{t,m}]$ so that $
\mathbb P_{\mathrm{cal}} [p^{v, \star}_m \in [L^v_{t,m}, U^v_{t,m}]] \geq 1 - \frac{\beta^v_{\mathrm{DRO}}}{M_v}
$. 
Following the standard equal-tailed construction, the Clopper--Pearson endpoints are then \begin{equation}
\begin{aligned}
L_{t,m}^{v}
&=
\begin{cases}
0,
& n^v_{t,m} = 0,\\[1mm]
F^{-1}_{\operatorname{Beta}(n_{t,m}^v,g^v_t-n^v_{t,m}+1)}
\left(\dfrac{\beta_{\mathrm{DRO}}^v}{2M_{v}}\right),
& n^v_{t,m} > 0,
\end{cases}
\\[2mm]
U_{t,m}^{v}
&=
\begin{cases}
1,
& n_{t,m}^v = g^v_t,\\[1mm]
F^{-1}_{\operatorname{Beta}(n_{t,m}^v+1,g^v_t-n_{t,m}^v)}
\left(1-\dfrac{\beta^v_{\mathrm{DRO}}}{2M_{v}}\right),
& n_{t,m}^v < g^v_t.
\end{cases}
\end{aligned}
\end{equation} 

Having established the Clopper-Pearson intervals, we now take the intersection of the Clopper-Pearson box $
\prod_{m=1}^{M_v}
[L_{t,m}^v,U_{t,m}^v],
$ and the mode distribution simplex $
\Delta_{\mathcal{M}_v}
=
\left\{
p\in\mathbb R^{M_v}:
p_m\ge0,\;
\sum_mp_m=1
\right\}
$. This allows us to determine which distributions within the simplex meet the confidence interval bounds. We define the per-obstacle confidence polytope at time $t$ as \[
\mathcal{C}^v_t = \{ p \in \Delta_{\mathcal{M}_v} : L^v_{t,m} \leq p_m \leq U^v_{t,m}, m =1 \dots M_v\}.
\]  For \(g_t^v\ge1\), the empirical frequency vector
\[
\bar p_{t,m}^v:=\frac{n_{t,m}^v}{g_t^v}
\]
lies in the corresponding Clopper--Pearson intervals and satisfies
\(\sum_m\bar p_{t,m}^v=1\); hence
\(\mathcal C_t^v\neq\varnothing\).

We then determine the simultaneous coverage event across modes. Let $E_{t,m}^{v} = \{ p_m^{v, \star} \notin [L_{t,m}^v, U^v_{t,m}]\}$ be the event that the mode $m$ is not within the confidence interval. Then if $p^{v, \star} \notin \mathcal{C}^v_t$, we know that at least one marginal interval does not hold. Therefore, $\{p^{v, \star} \notin \mathcal{C}_t^v\} = \bigcup_{m=1}^{M_v} E_{t,m}^{v}$. By the union bound, we then see $\mathbb P_{\rm cal}
[
p^{v,\star}\notin\mathcal C_t^v
]
=
\beta_{\mathrm{DRO}}^v$. Thus \[
\mathbb P_{\mathrm{cal}} [p^{v, \star} \in \mathcal{C}^v_t ] \geq 1 - \beta_{\mathrm{DRO}}^v.
\]

Finally, because we require that all obstacles be simultaneously covered we use a Bonferroni adjustment. Given $\sum_{v=1}^V \beta_{\mathrm{DRO}}^v \leq \beta_{\mathrm{DRO}}$, another application of the union bound across obstacles yields $\mathbb P_{\mathrm{cal}} [ \exists v : p^{v, \star} \notin \mathcal{C}^v_t] \leq \sum_{v=1}^V \beta_{\mathrm{DRO}}^v \leq \beta_{\mathrm{DRO}}$. Equivalently, \[
\mathbb P_{\rm cal}
\left[
p^{v,\star}\in\mathcal C_t^v
\quad\forall v
\right]
\ge
1-\beta_{\mathrm{DRO}}.
\] 

We summarize the simultaneous finite-sample coverage established
above in the following proposition, which will be used in the
distribution-transfer theorem.

\begin{prop}[Finite-sample confidence-set coverage]
\label{prop:ambiguity_coverage}
Let
\[
\mathcal C_t
=
\mathcal C_t^1\times\cdots\times\mathcal C_t^V.
\]
Under the calibration assumptions above,
\[ \mathbb P_{\rm cal} [ \mathbf p^\star\in\mathcal C_t ]
\ge
1-\beta_{\rm DRO}.
\]
\end{prop}

\noindent\emph{Proof.}
This follows directly from the simultaneous Clopper--Pearson
coverage and
\(\sum_v\beta_{\rm DRO}^v\le\beta_{\rm DRO}\).
\hfill\(\square\)

\subsection{Coordinate envelope}

We now construct a coordinate envelope from \(\mathcal C_t^v\). For each mode define
\[
u_{t,m}^v
:=
\sup_{p\in\mathcal C_t^v}p_m.
\]
Because \(\mathcal C_t^v\) is the intersection of the simplex with
coordinate intervals, this quantity has the closed form
\[
u_{t,m}^v
=
\min
\left\{
U_{t,m}^v,\;
1-\sum_{j\neq m}L_{t,j}^v
\right\}.
\]
Therefore every statistically covered distribution satisfies
\[
p_m^v
\le
u_{t,m}^v.
\]

\subsection{Domination envelope and candidate distribution construction}
\label{sec:domination_inclusion}
\label{sec:constrained_distribution_construction}

\begin{prop}[Domination envelope]
\label{prop:domination_envelope}
For any \(q_t^v\in\Delta_{\mathcal M_v}\), define
\[
\zeta_t^v(q_t^v)
:=
\max_{m:u_{t,m}^v>0}
\frac{u_{t,m}^v}{q_{t,m}^v},
\]
with \(\zeta_t^v(q_t^v)=+\infty\) if
\(q_{t,m}^v=0<u_{t,m}^v\).
Then every \(p_t^v\in\mathcal C_t^v\) satisfies
\[
p_{t,m}^v
\le
\zeta_t^v(q_t^v)q_{t,m}^v,
\qquad
\forall m,
\]
and \(\zeta_t^v(q_t^v)\) is the smallest scalar with this property.

Define
\[
\widehat\zeta_t^v
:=
\max_{m:u_{t,m}^v>0}
\frac{u_{t,m}^v}{\widehat p_{t,m}^v}
\]
and the nominal-domination candidate set
\[
\mathcal Q_t^v
:=
\left\{
q\in\Delta_{\mathcal M_v}:
q_m
\ge
\frac{u_{t,m}^v}{\widehat\zeta_t^v},
\quad
\forall m
\right\}.
\]
Any sampling law with finite \(\zeta_t^v(q)\) admits the transfer
bound below; \(\mathcal Q_t^v\) restricts attention to laws whose
domination factor is no larger than the nominal factor
\(\widehat\zeta_t^v\).

Then
\(\widehat p_t^v\in\mathcal Q_t^v\), so
\(\mathcal Q_t^v\neq\varnothing\), and every
\(q_t^v\in\mathcal Q_t^v\) satisfies
\[
\zeta_t^v(q_t^v)
\le
\widehat\zeta_t^v.
\]
\end{prop}

\noindent\emph{Proof.}
For \(p_t^v\in\mathcal C_t^v\),
\[
p_{t,m}^v
\le
u_{t,m}^v
\le
\zeta_t^v(q_t^v)q_{t,m}^v.
\]
Conversely, any uniform domination factor \(\gamma\) must satisfy
\(u_{t,m}^v\le\gamma q_{t,m}^v\), hence
\(\gamma\ge\zeta_t^v(q_t^v)\).
Finally, full support of \(\widehat p_t^v\) and the definition of
\(\widehat\zeta_t^v\) give
\(\widehat p_t^v\in\mathcal Q_t^v\); the coordinate lower bounds
defining \(\mathcal Q_t^v\) imply
\(\zeta_t^v(q_t^v)\le\widehat\zeta_t^v\).
\hfill\(\square\)

\paragraph{Minimum domination benchmark.}
If \(q\) is allowed to range over the full simplex, then
\[
\underline\zeta_t^v
:=
\min_{q\in\Delta_{\mathcal M_v}}
\zeta_t^v(q)
=
\sum_{m=1}^{M_v}u_{t,m}^v,
\]
with minimizer
\[
q_{t,m}^{v,\mathrm{env}}
=
\frac{u_{t,m}^v}
{\sum_j u_{t,j}^v}.
\]
Indeed, \(u_{t,m}^v\le\zeta q_m\) implies
\(\sum_m u_{t,m}^v\le\zeta\), with equality at
\(q^{v,\mathrm{env}}\).

Any additional restriction on the sampling distribution can only
weakly increase the minimum attainable domination factor above
\(\underline\zeta_t^v\). Thus
\(\underline\zeta_t^v\) provides a method-independent lower benchmark
for the certificate tightening associated with any subsequently
selected sampling law.

The trajectory-law transfer itself is agnostic to how a particular
sampling distribution is selected, provided its domination factor is
finite. We use \(\mathcal Q_t^v\) to restrict attention to candidates
whose domination factor is no larger than that obtained under nominal
sampling.

For the certification result below, let \(\mathcal F_t\) denote all information used to construct \(\mathbf q_t^\star\), \(\zeta_t\), the scenario count \(S\), and the support cap \(\bar n\) before drawing the current certification scenarios. Conditional on \(\mathcal F_t\), these quantities are fixed and the current scenarios are drawn independently from \(\mathcal P_{\mathbf q_t^\star}\).
Section~\ref{sec:risk_aware_wass} later gives one particular
risk-aware selection rule.

\subsection{Trajectory-law and violation domination}
\label{sec:domination_transfer}
\begin{prop}[Trajectory-law and violation domination]
\label{prop:law_domination}
Let
\[
\zeta_t
:=
\prod_{v=1}^V
\zeta_t^v(q_t^{v,\star}).
\]
Under conditional independence across obstacles, every
\[
\mathbf p
\in
\mathcal C_t
:=
\mathcal C_t^1\times\cdots\times\mathcal C_t^V
\]
satisfies
\[
\mathcal P_{\mathbf p}(A)
\le
\zeta_t
\mathcal P_{\mathbf q_t^\star}(A)
\]
for every measurable joint trajectory event \(A\). Consequently,
\[
\mathrm{V}_{\mathcal P_{\mathbf p}}(\vartheta)
\le
\zeta_t
\mathrm{V}_{\mathcal P_{\mathbf q_t^\star}}(\vartheta).
\]
Hence
\[
\mathrm{V}_{\mathcal P_{\mathbf q_t^\star}}(\vartheta)
\le
\frac{\epsilon}{\zeta_t}
\quad\Longrightarrow\quad
\mathrm{V}_{\mathcal P_{\mathbf p}}(\vartheta)
\le
\epsilon
\]
for every \(\mathbf p\in\mathcal C_t\).
\end{prop}

\noindent\emph{Proof.}
Proposition~\ref{prop:domination_envelope} gives
\[
p_{m_v}^v
\le
\zeta_t^v(q_t^{v,\star})
q_{t,m_v}^{v,\star}
\]
for each obstacle. Multiplying across independent obstacles and
summing over the joint mode mixture gives the measure inequality.
Taking
\[
A=\{\xi:\vartheta\notin\Theta_\xi\}
\]
gives the violation bound.
\hfill\(\square\)

\subsection{Safe-Horizon transfer}

\paragraph{Conservatism of obstacle-wise transfer.}
For every obstacle,
\[
\zeta_t^v(q_t^v)
\ge
\sum_{m=1}^{M_v}u_{t,m}^v
\ge
1.
\]
Consequently,
\[
\zeta_t
\ge
\prod_{v=1}^{V}
\left(
\sum_{m=1}^{M_v}u_{t,m}^v
\right)
\ge
1.
\]
Thus the obstacle-wise transfer factors compound multiplicatively.
If the individual factors remain bounded away from one, the joint
tightening can grow rapidly with \(V\). 
\begin{thm}[Safe-Horizon distribution transfer]
\label{thm:safe_horizon_transfer}

Fix a replanning instant \(t\). Let \(\mathcal F_t\) be the pre-certification information defined above. Suppose
\[
S\in\mathbb N,
\qquad
\bar n\in\{0,\ldots,S\}
\]
are \(\mathcal F_t\)-measurable and therefore fixed before drawing the
current certification scenarios. Conditional on \(\mathcal F_t\), let
\[
\Xi_t
:=
(\xi_t^{(1)},\ldots,\xi_t^{(S)}),
\qquad
\xi_t^{(i)}\mid\mathcal F_t
\overset{\mathrm{iid}}{\sim}
\mathcal P_{\mathbf q_t^\star}.
\]

Let \(\mathcal C_t\) be constructed as in
Proposition~\ref{prop:ambiguity_coverage}, and suppose
\(q_t^{v,\star}\in\mathcal Q_t^v\) for every obstacle \(v\), so that
Proposition~\ref{prop:law_domination} applies. Let
\(n_{\Xi_t}\) denote the support size of the SH-MPC solution
\(\vartheta_{\Xi_t}^\star\), and suppose \(S\) and \(\bar n\) are
selected such that
\[
\epsilon_S(\bar n)
\le
\frac{\epsilon}{\zeta_t}.
\]

If Safe-Horizon certifies only solutions satisfying
\[
n_{\Xi_t}\le\bar n,
\]
then
\[
\mathbb P
\left[
\mathrm{V}_{\mathcal P_{\mathbf p^\star}}
(\vartheta_{\Xi_t}^\star)
>
\epsilon,\;
n_{\Xi_t}\le\bar n
\right]
\le
\beta_{\mathrm{DRO}}
+
\beta_{\mathrm{SH}}.
\]
\end{thm}

Here the probability is taken jointly over the calibration data and
the current certification-scenario draw.

\noindent\emph{Proof.}
Let
\[
E_{\rm cal}
:=
\{\mathbf p^\star\in\mathcal C_t\}.
\]
By Proposition~\ref{prop:ambiguity_coverage},
\[
\mathbb P(E_{\rm cal}^c)
\le
\beta_{\rm DRO}.
\]
On \(E_{\rm cal}\),
Proposition~\ref{prop:law_domination} gives
\[
\mathrm{V}_{\mathcal P_{\mathbf p^\star}}
(\vartheta_{\Xi_t}^\star)
\le
\zeta_t
\mathrm{V}_{\mathcal P_{\mathbf q_t^\star}}
(\vartheta_{\Xi_t}^\star).
\]
For \(n_{\Xi_t}\le\bar n\),
\[
\epsilon_S(n_{\Xi_t})
\le
\epsilon_S(\bar n)
\le
\epsilon/\zeta_t.
\]
Conditional on \(\mathcal F_t\), Safe-Horizon bounds failure of this
certificate by \(\beta_{\rm SH}\); averaging over \(\mathcal F_t\)
preserves the bound. A union bound with \(E_{\rm cal}^c\) completes
the proof.
\hfill\(\square\)

Theorem~\ref{thm:safe_horizon_transfer} is a pointwise
replanning-time certificate, not a time-uniform closed-loop safety
guarantee. Simultaneous coverage over multiple replanning instants
would require an allocation of failure probability over time or a
time-uniform calibration procedure. In addition, the product form of
\(\zeta_t\) can make the transferred certificate increasingly
conservative as the number of simultaneously modeled obstacles grows.

\section{Risk-Aware Wasserstein Construction}

In the preceding section, the distribution-transfer certificate was derived without requiring Wasserstein geometry. We now introduce a Wasserstein transport constraint and a risk-aware objective as implementation-level restrictions for selecting a particular certificate-valid sampling law. These restrictions affect the properties of the selected distribution but are not required for validity of Theorem~\ref{thm:safe_horizon_transfer}. We refer to the original SH-MPC formulation \cite{degroot2023SBMotionBounded} and our accompanying implementation at \url{https://github.com/stephen-crawford/DRO_SHMPCC/tree/8439eb436647dc947b9e3006efb4d6e6f904f5ae} for standard SH-MPC details.

\subsection{Wasserstein budget}

We introduce Wasserstein geometry as a mechanism to constrain how the reweighted distribution may move probability mass away from the nominal belief.

Using the geometry-dependent ground cost \(D_t^v\) defined in
Section~\ref{sec:ground_cost}, for discrete distributions
\(a,b\in\Delta_{\mathcal M_v}\), define
\[
W_D(a,b)
:=
\min_{\lambda\ge0}
\sum_{i,j}D_{ij}\lambda_{ij}
\]
subject to
\[
\sum_j\lambda_{ij}=a_i,
\qquad
\sum_i\lambda_{ij}=b_j.
\]
We select a transport budget
\begin{equation}\label{eq:radius_calibration}
\rho_t^v
=
\max_{p\in\mathcal C_t^v}
W_{D_t^v}(\widehat p_t^v,p).
\end{equation}
Thus every statistically plausible mode law lies within transport
cost \(\rho_t^v\) of the nominal belief. Importantly, this
Wasserstein neighborhood is used below as a geometry-aware design
constraint on the selected sampling law \(q_t^v\); the domination
envelope itself is still computed directly from \(\mathcal C_t^v\).

For fixed \(\widehat p_t^v\) and \(D_t^v\),
\(W_{D_t^v}(\widehat p_t^v,p)\) is convex and piecewise linear in
\(p\). Hence its maximum over the compact polytope
\(\mathcal C_t^v\) is attained at an extreme point. For the
low-dimensional mode sets considered here, we compute
\(\rho_t^v\) exactly by enumerating the vertices of
\(\mathcal C_t^v\) and solving the corresponding discrete
optimal-transport problem at each vertex.

The resulting radius is used only as a geometry-aware transport
budget for selecting \(q_t^v\); it does not enlarge the confidence
set used in the distribution-transfer certificate.

\subsection{Ground cost}\label{sec:ground_cost}
The transfer result of Section~\ref{sec:dist_transfer} does not
require a Wasserstein metric. Wasserstein geometry is instead used to
regularize how probability mass may be reallocated when constructing
the scenario-sampling distribution. In our switching-dynamics setting,
mode labels such as \texttt{turn left} and \texttt{turn right} do not
themselves define a meaningful geometry. For the ground-cost
construction, let
\(\mu_{m,k}^v\in\mathbb R^2\) and
\(\Sigma_{m,k}^v\in\mathbb S_+^2\)
denote the mean and covariance of the positional marginal of the
mode-\(m\) Gaussian prediction at horizon step \(k\).
We define
\begin{equation}\label{eq:ground_cost}
D_{ij,t}^{v}
=
\frac{1}{N}
\sum_{k=1}^{N}
W_2\!\left(
\mathcal{N}(\mu_{i,k}^{v},\Sigma_{i,k}^{v}),
\mathcal{N}(\mu_{j,k}^{v},\Sigma_{j,k}^{v})
\right).
\end{equation} Although the mode dynamics \((A_m^v,b_m^v,G_m^v)\) are fixed, the induced predictions are conditioned on the current obstacle state. Consequently, \(D_t^v\) may vary across replanning instants.

Because $W_2^2$ does not satisfy the triangle inequality, we use $W_2$ as the ground cost. Averaging the per-step $W_2$ distances preserves nonnegativity, symmetry, and the triangle inequality. The resulting $D^{v}_t$ therefore provides a geometry-consistent transport cost over the mode predictions; it is a metric on the induced positional Gaussian predictions and may be a pseudometric on the mode labels if two distinct modes generate identical positional predictions. 

\subsection{Risk score}\label{sec:risk_aware_wass}

The distribution-transfer certificate does not depend on a collision
risk surrogate; it is valid for any sampling law satisfying the
domination constraints derived above. We nevertheless require a rule
for selecting among these certificate-valid laws. We therefore assign
each mode a nonnegative score \(r_{t,m}^v\) that upper-bounds a
trajectory-level penetration-risk quantile and use this score only to
preferentially sample modes judged dangerous.

For obstacle \(v\), mode \(m\), and horizon step
\(k\in\{1,\ldots,N\}\), suppressing the current replanning index
\(t\) where unambiguous, let the predicted obstacle position satisfy
\[
X_k^{v,m}
\sim
\mathcal N
\left(
\mu_{m,k}^v,
\Sigma_{m,k}^v
\right).
\]

The ego footprint is represented by \(n_d\) discs. Let
\(\overline p_{d,k}\) denote the frozen center of disc \(d\) at
horizon step \(k\), obtained from the shifted previous solution or
cold-start reference, and let
\[
R^v:=r_d+r_{\rm obs}^v+r_{\rm offset}.
\]
For mode \(m\), define the penetration depth
\[
\mathbf \delta_{d,k}^{v,m}
=
\left[
R^v-
\|X_k^{v,m}-\overline p_{d,k}\|_2
\right]_+,
\ \ \ 
\mathbf \delta_t^{v,m}
=
\max_{d,k}\mathbf \delta_{d,k}^{v,m}.
\]

Because the Euclidean norm of a Gaussian vector is not generally
Gaussian, we project along the frozen unit direction
\[
\hat n_{d,k}^{v,m}
=
\frac{\mu_{m,k}^v-\overline p_{d,k}}
{\|\mu_{m,k}^v-\overline p_{d,k}\|_2},
\]
using a fixed unit-vector fallback when the denominator is below a specified numerical safety parameter
\(\varepsilon_{\min} > 0\). Define the projected penetration variable
\[
Z_{d,k}^{v,m}
:=
R^v
-
(\hat n_{d,k}^{v,m})^\top
\left(
X_k^{v,m}-\overline p_{d,k}
\right).
\]
Since \(X_k^{v,m}\) is Gaussian,
\(Z_{d,k}^{v,m}\) is Gaussian with
\[
\mu_{Z,d,k}^{v,m}
=
R^v-
(\hat n_{d,k}^{v,m})^\top
(\mu_{m,k}^v-\overline p_{d,k}),
\]
\[
(\sigma_{Z,d,k}^{v,m})^2
=
(\hat n_{d,k}^{v,m})^\top
\Sigma_{m,k}^v
\hat n_{d,k}^{v,m}.
\]
Since \(\hat n_{d,k}^{v,m}\) is a unit vector,
Cauchy--Schwarz gives the pointwise bound
\[
\mathbf \delta_{d,k}^{v,m}
\le
[Z_{d,k}^{v,m}]_+.
\]

We use the left-quantile convention
\[
\operatorname{VaR}_{\alpha}(Y)
:=
\inf\{y\in\mathbb R:F_Y(y)\ge\alpha\}.
\]

For a desired trajectory-level confidence
\(\alpha\in[1/2,1)\) and prescribed standard-deviation floor
\(\sigma_{\min}\ge0\), define
\[
\alpha'
=
1-\frac{1-\alpha}{Nn_d},
\qquad
\widetilde\sigma_{Z,d,k}^{v,m}
=
\max\{\sigma_{Z,d,k}^{v,m},\sigma_{\min}\},
\]
and the mode-risk score
\begin{equation}\label{eq:risk_score}
r_{t,m}^v
=
\max_{d,k}
\left[
\mu_{Z,d,k}^{v,m}
+
\Phi^{-1}(\alpha')
\widetilde\sigma_{Z,d,k}^{v,m}
\right]_+.
\end{equation}

\begin{lemma}[Trajectory-level penetration-risk bound]
\label{lemma:risk_score_bound}
The score in \eqref{eq:risk_score} satisfies
\[
\operatorname{VaR}_{\alpha}
(\mathbf \delta_t^{v,m})
\le
r_{t,m}^v.
\]
\end{lemma}

\noindent\emph{Proof.}
Let
\[
\tau
=
\max_{d,k}
\operatorname{VaR}_{\alpha'}
([Z_{d,k}^{v,m}]_+).
\]
The pointwise projection bound and a union bound give
\[
\mathbb P[\mathbf \delta_t^{v,m}>\tau]
\le
Nn_d(1-\alpha')
=
1-\alpha.
\]
For Gaussian \(Z\),
\[
\operatorname{VaR}_{\alpha'}([Z]_+)
=
[\mu_Z+\Phi^{-1}(\alpha')\sigma_Z]_+.
\]
Since \(\alpha'\ge1/2\) and
\(\widetilde\sigma_Z\ge\sigma_Z\), replacing \(\sigma_Z\) by
\(\widetilde\sigma_Z\) can only increase the bound.
\hfill\(\square\)

Lemma~\ref{lemma:risk_score_bound} justifies the ordering used to
select a sampling law; the risk score does not enter the
distribution-transfer guarantee of
Theorem~\ref{thm:safe_horizon_transfer}.

Having defined \(r_{t,m}^v\), we now impose the Wasserstein design
budget on the certificate-valid set:
\[
\mathcal Q_{t,W}^v
:=
\left\{
q\in\mathcal Q_t^v:
W_{D_t^v}(\widehat p_t^v,q)
\le
\rho_t^v
\right\}.
\]
Because \(\widehat p_t^v\in\mathcal Q_t^v\) and
\(W_{D_t^v}(\widehat p_t^v,\widehat p_t^v)=0\),
the set \(\mathcal Q_{t,W}^v\) is nonempty.

Let
\[
\mathcal H_t^v
:=
\{m\in\mathcal M_v:r_{t,m}^v\ge\tau_r\},
\qquad
\tau_r>0,
\]
and define
\[
\mathcal Q_{t,\mathrm{risk}}^v
:=
\left\{
q\in\mathcal Q_{t,W}^v:
q_m\ge\widehat p_{t,m}^v,\;
\forall m\in\mathcal H_t^v
\right\}.
\]
Finally,
\[
q_t^{v,\star}
\in
\arg\max_{q\in\mathcal Q_{t,\mathrm{risk}}^v}
\langle q,r_t^v\rangle.
\]

\paragraph{Properties of the risk-aware sampling law.}
Because
\(\widehat p_t^v\in\mathcal Q_{t,\mathrm{risk}}^v\),
the feasible set is nonempty. Moreover,
\(\mathcal Q_{t,\mathrm{risk}}^v\) is a closed subset of the compact
simplex \(\Delta_{\mathcal M_v}\), and is therefore compact.
Since \(q\mapsto\langle q,r_t^v\rangle\) is continuous, an optimizer
\(q_t^{v,\star}\) exists. Any optimizer satisfies
\[
\langle q_t^{v,\star},r_t^v\rangle
\ge
\langle\widehat p_t^v,r_t^v\rangle,
\qquad
\zeta_t^v(q_t^{v,\star})
\le
\widehat\zeta_t^v,
\]
with
\[
q_{t,m}^{v,\star}
\ge
\widehat p_{t,m}^v,
\qquad
m\in\mathcal H_t^v.
\]
These properties follow respectively from nominal feasibility,
\(q_t^{v,\star}\in\mathcal Q_t^v\), and the dangerous-mode
constraints. Importantly, the risk score determines which
certificate-valid sampling law is selected; it is not required for
the validity of the distribution-transfer certificate.

\section{Results}\label{results}

We evaluate two aspects of the proposed distribution-transfer
construction. First, we isolate how the scenario-sampling law affects
certificate tightening. Second, we evaluate the resulting controller
using independent trajectories drawn from known true mode laws. Unless otherwise stated, experiments use \(\epsilon=0.05\), \(\beta_{\rm SH}=0.01\), \(\beta_{\rm DRO}=0.05\), $\alpha = 0.95$, $\tau_r = 0.01m$, $\sigma_{\min} = 10^{-6}$, support cap \(\bar n=6\), and scenario-removal budget \(R_{\rm rem}=2\). Note that \(S\) and \(\bar n\) are selected before drawing the scenario set.

\subsection{Sampling-law ablation}

\begin{table}[t]
\centering
\caption{Sampling-law ablation for \(M=5\), \(g=100\), over 100
independent calibration instances. Here
\(\underline{\zeta}=\sum_m u_m\) is the minimum attainable domination
factor, \(\Delta r=\langle q,r\rangle-\langle\widehat p,r\rangle\),
and \(\Delta q_{\rm dang}\) is the change in total probability assigned
to dangerous modes relative to nominal sampling.}
\label{tab:certificate_ablation}
\small
\setlength{\tabcolsep}{3.2pt}
\resizebox{\columnwidth}{!}{%
\begin{tabular}{lccccc}
\toprule
Sampling law &
\(\zeta/\underline{\zeta}\) &
\(\zeta\) &
Required \(S\) &
\(\Delta r\) &
\(\Delta q_{\rm dang}\) \\
\midrule
Envelope optimum \(q^{\rm env}\)
    & 1.000 & 1.573 & 1931 & 0.0149 & -0.0356 \\
Wasserstein min-\(\zeta\) \(q^W\)
    & 1.000 & 1.573 & 1931 & 0.0149 & -0.0356 \\
Nominal \(\widehat p\)
    & 1.420 & 2.225 & 2901 & 0      & 0       \\
Risk-aware \(q^\star\)
    & 1.420 & 2.225 & 2901 & 0.0458 & +0.0624 \\
\bottomrule
\end{tabular}%
}
\end{table}

Table~\ref{tab:certificate_ablation} compares an envelope-optimal law
\(q^{\rm env}\), a Wasserstein-constrained minimum-domination law
\(q^W\), nominal sampling \(\widehat p\), and the proposed risk-aware
law \(q^\star\). The first two serve as offline references: \(q^{\rm env}\)
attains the minimum possible domination factor, while \(q^W\) tests
whether imposing the Wasserstein transport constraint increases this
minimum.

For \(M=5\), \(g=100\), and 100 independent calibration instances,
both \(q^{\rm env}\) and \(q^W\) attain median
\(\zeta=1.573\), requiring \(1931\) scenarios. Thus, the Wasserstein
constraint is non-binding with respect to the minimum attainable
domination factor in these instances. Nominal sampling and
\(q^\star\) both yield median \(\zeta=2.225\) and \(S=2901\), but
\(q^\star\) increases the median probability assigned to dangerous
modes by \(0.0624\) and the risk objective by \(0.0458\). The proposed
law therefore biases sampling toward dangerous modes without
additional certificate tightening relative to nominal sampling.
Conversely, minimizing \(\zeta\) alone reduces dangerous-mode mass by
\(0.0356\), illustrating the tradeoff between certificate tightness
and risk-directed sampling.

\subsection{Held-out true-law validation}

\begin{figure*}[!t]
    \centering
    \includegraphics[width=0.92\textwidth]
        {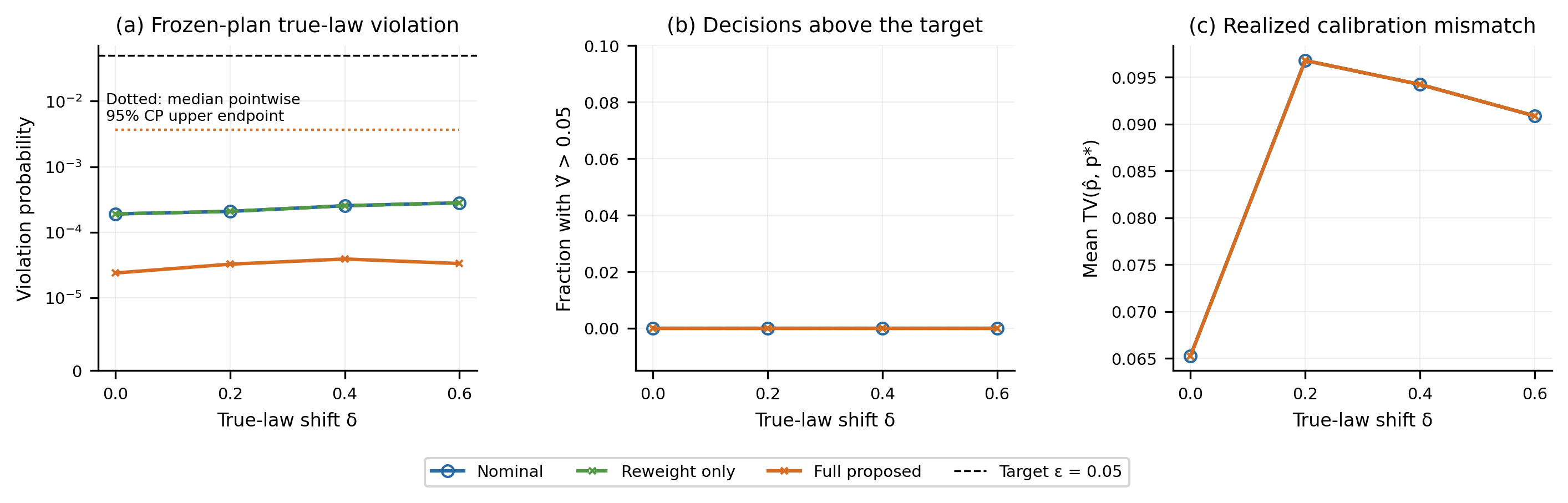}
    \caption{Held-out true-law validation. Solid curves show equal-seed-weighted mean violation estimates over 30 paired seeds; dotted curves show median pointwise 95\% Clopper–Pearson upper endpoints from 1000 independent held-mode validation trajectories per returned decision.}
    \label{fig:true_law_validation}
\end{figure*}

The closed-loop study separates probability reweighting from
transfer-induced certificate tightening. We use a straight,
one-obstacle, three-mode setting with \(g=25\) calibration
observations. Let \(m_{\rm risk}\) denote the mode with the largest
frozen initial risk score and define $ p^\star(\gamma)
=
(1-\gamma)(0.8,0.15,0.05)
+
\gamma e_{m_{\rm risk}}$ where
$\gamma\in\{0,0.2,0.4,0.6\}$.
Calibration observations are redrawn from the corresponding
\(p^\star(\gamma)\); thus the experiment evaluates finite-sample
mode-probability estimation error under different true laws rather
than an intentionally stale nominal model. The nominal
controller samples from \(\widehat p\) using the standard scenario
budget \(S_0\), the reweight-only controller samples from \(q^\star\) using
the same \(S_0\), and the full transfer controller samples from \(q^\star\). 
The full-transfer controller samples from \(q^\star\) and increases \(S\) to satisfy \(\epsilon_S(\bar n)\le\epsilon/\zeta_t\). We run 30 paired trials per law and controller and validate each returned decision using 1000
independent held-mode trajectories from the corresponding true law; results are shown in Fig.~\ref{fig:true_law_validation}.

All \(360\) closed-loop trials completed without an observed physical collision. The equal-seed-weighted mean held-out violation estimate was \(2.34\times10^{-4}\) for nominal and reweight-only sampling and \(3.23\times10^{-5}\) for full transfer, approximately \(86\%\) lower.

Because reweight-only changes the sampling law without changing the
scenario budget, its agreement with nominal sampling suggests that
the lower held-out estimate is not attributable to reweighting alone.
Rather, it accompanies the additional scenario tightening used by the
full-transfer controller. This comes at substantial computational
cost: nominal and reweight-only use \(S_0=1123\), whereas full transfer
requires as many as \(14828\) scenarios. Median measured solve times
increase from \(13.5\)--\(13.9\) ms to \(72.5\)--\(156.3\) ms under
the four-process experimental setup, excluding held-out validation.
These results characterize the central tradeoff of the method:
finite-sample distribution transfer can substantially increase the
scenario requirement even when the sampling-law optimization itself
does not increase the domination factor relative to nominal sampling.


\section{Conclusion}\label{conclusion}

We developed a finite-sample distribution-transfer certificate for
Safe-Horizon MPC under categorical mode-probability uncertainty.
Clopper--Pearson confidence sets and a multiplicative domination bound transfer the Safe-Horizon violation guarantee from a selected scenario-sampling law to the unknown true mode law: the probability of returning an accepted solution whose true-law violation exceeds \(\epsilon\) is bounded by \(\beta_{\rm DRO}+\beta_{\rm SH}\). Wasserstein geometry and a trajectory-risk surrogate are used
separately to select a sampling law that prioritizes dangerous modes
while controlling the domination factor. Experiments show that this
risk-aware selection can increase dangerous-mode sampling without
increasing median certificate tightening relative to nominal sampling,
while enforcing the complete transfer condition yields lower held-out
violation estimates at the cost of substantially larger scenario
counts and solve times. The principal limitation is the conservatism
of the multiplicative transfer, particularly as obstacle-wise factors
compound. Future work will investigate tighter multi-obstacle transfer
bounds and extensions to state-dependent or Markov mode models.


\section{Acknowledgements}\label{acknowledgements}

Throughout this publication, basic proof formulation was assisted by ChatGPT GPT-5.6-sol. Code implementation, execution, and figure generation used ChatGPT GPT-6-astra with manual review by the authors. Code implementation and execution also made use of Claude-5-Sonnet. All math was verified manually and reformatted by the authors. 

This material is based upon work supported by the National Science Foundation CISE Graduate Fellowships under Grant No. 2313998. Any opinions, findings, and conclusions or recommendations expressed in this material are those of the author(s) and do not necessarily reflect the views of the National Science Foundation.

\bibliographystyle{IEEEtran}
\bibliography{tex}
\end{document}